\documentclass[11pt]{article}

\usepackage{latexsym,natbib}
\usepackage[latin1]{inputenc}
\usepackage{amsmath, amssymb, graphicx}
\usepackage{algorithm}
\usepackage[noend]{algorithmic}
\usepackage{paralist}
\usepackage{dsfont}

\newcounter{mathcounter_d}
\newcounter{mathcounter_l}
\newcounter{mathcounter_c}
\newcounter{mathcounter_t}
\newenvironment{definition}{\refstepcounter{mathcounter_d} \begin{trivlist} \item[\hskip \labelsep {\bfseries Definition \arabic{mathcounter_d}.\enspace}] \it}{\end{trivlist}}
\newenvironment{lemma}{\refstepcounter{mathcounter_l} \begin{trivlist} \item[\hskip \labelsep {\bfseries Lemma \arabic{mathcounter_l}.\enspace}] \it}{\end{trivlist}}
\newenvironment{corollary}{\refstepcounter{mathcounter_c} \begin{trivlist} \item[\hskip \labelsep {\bfseries Corollary \arabic{mathcounter_c}.\enspace}]}{\end{trivlist}}
\newenvironment{theorem}{\refstepcounter{mathcounter_t} \begin{trivlist} \item[\hskip \labelsep {\bfseries Theorem \arabic{mathcounter_t}.\enspace}]}{\end{trivlist}}

\newcommand{\qed}{\nobreak \ifvmode \relax \else \ifdim\lastskip<1.5em \hskip-\lastskip \hskip1.5em plus0em minus0.5em \fi \nobreak \vrule height0.4em width0.5em depth0.25em\fi}
\newenvironment{proof}[1][Proof]{\begin{trivlist} \item[\hskip \labelsep {\bfseries #1}]}{\hfill\qed\end{trivlist}}

\newcommand{\N}{\mathbb{N}}
\newcommand{\R}{\mathbb{R}}

\newcommand{\Normal}{\mathcal{N}}
\newcommand{\Expectation}{\mathbb{E}}

\newcommand{\indicator}{\mathds{1}}
\newcommand{\Order}{\mathcal{O}}

\newcommand{\f}{\widehat{f}_{\Lambda}}

\newcommand{\closure}[1]{\overline{#1}}

\newcommand{\bmat}{\begin{pmatrix}}
\newcommand{\emat}{\end{pmatrix}}

\begin{document}

\title{Global Convergence of the (1+1) Evolution Strategy}
\author{Tobias Glasmachers\\
		Institute for Neural Computation\\
		Ruhr-University Bochum, Germany\\
		\texttt{tobias.glasmachers@ini.rub.de}}
\date{}

\maketitle

\begin{abstract}
We establish global convergence of the (1+1) evolution strategy, i.e.,
convergence to a critical point independent of the initial state.
More precisely, we show the existence of a critical limit point, using
a suitable extension of the notion of a critical point to measurable
functions.
At its core, the analysis is based on a novel progress guarantee for
elitist, rank-based evolutionary algorithms. By applying it to the
(1+1) evolution strategy we are able to provide an accurate
characterization of whether global convergence is guaranteed with full
probability, or whether premature convergence is possible.
We illustrate our results on a number of example applications ranging
from smooth (non-convex) cases over different types of saddle points and
ridge functions to discontinuous and extremely rugged problems.
\end{abstract}

%%%%%%%%%%%%%%%%%%%%%%%%%%%%%%%%%%%%%%%%%%%%%%%%%%%%%%%%%%%%
%%
\section{Introduction}
\label{section:introduction}

Global convergence of an optimization algorithm refers to convergence of
the iterates to a critical point independent of the initial state---in
contrast to local convergence, which guarantees this property only for
initial iterates in the vicinity of a critical point.%
\footnote{Some authors refer to global convergence as convergence to a
global optimum. We do not use the term in this sense.}
For example, many first order methods enjoy this
property~\citep{gilbert1992global}, while Newton's method does not. In
the realm of direct search algorithms, mesh adaptive search algorithms
are known to be globally convergent~\citep{torczon1997convergence}.

Evolution strategies (ES) are a class of randomized search heuristics
for direct search in $\R^d$. The (1+1)-ES is the maybe simplest such
method, originally developed by \cite{rechenberg:1973}.
A particularly simple variant thereof, which was first defined by
\cite{kern:2004}, is given in Algorithm~\ref{algorithm:oneplusoneES}.
Its state consists of a single parent individual $m \in \R^d$ and a step
size $\sigma > 0$. It samples a single offspring $x \in \R^d$ per
generation from the isotropic multivariate normal distribution
$\Normal(m, \sigma^2 I)$ and applies (1+1)-selection, i.e., it keeps the
better of the two points. Here, $I \in \R^{d \times d}$ denotes the
identity matrix.
The standard deviation $\sigma > 0$ of the sampling distribution, also
called \emph{global step size}, is adapted online. The mechanism
maintains a fixed success rate usually chosen as $1/5$, in accordance
with Rechenbergs original approach. It is discussed in more detail in
section~\ref{section:step-size-control}. In effect, step size control
enables linear convergence on convex quadratic functions
\citep{jaegerskuepper2006how}, and therefore locally linear convergence
on twice differentiable functions. In contrast, algorithms without step
size adaptation can converge as slowly as pure random search
\citep{hansen2015evolution}. Furthermore, being rank-based methods, ESs
are invariant to strictly monotonic transformations of objective values.
ESs tend to be robust and suitable for solving difficult problems
(rugged and multimodal fitness landscapes), a capacity that is often
attributed to invariance properties.

Although the (1+1)-ES is the oldest evolution strategy in existence, we
do not yet fully understand how generally it is applicable. In this
paper we cast this open problem into the question on which functions the
algorithm will succeed to locate a local optimum, and on which functions
it may converge prematurely, and hence fail. We aim at an as complete as
possible characterization of these different cases.

The covariance matrix adaptation evolution strategy (CMA-ES) by
\cite{hansen:2001} and its many variants mark the state-of-the-art. The
algorithm goes beyond the simple (1+1)-ES in many ways: it uses
non-elitist selection with a population, it adapts the full covariance
matrix of its sampling distribution (effectively resembling second order
methods), and it performs temporal integration of direction information
in the form of evolution paths for step size and covariance matrix
adaptation. Still, its convergence order on many relevant functions is
linear, and that is thanks to the same mechanism as in the (1+1)-ES,
namely step size adaptation.

To date, convergence guarantees for ESs are scarce. Some results exist
for convex quadratic problems, which essentially implies local
convergence on twice continuously differentiable functions.
In this situation it is natural to start with the simplest ES, which is
arguably the (1+1)-ES. The variant defined by \cite{kern:2004} is given
in Algorithm~\ref{algorithm:oneplusoneES}; it is discussed in detail in
section~\ref{section:step-size-control}.

\begin{algorithm}
\caption{(1+1)-ES}
\label{algorithm:oneplusoneES}
\begin{algorithmic}[1]
\STATE{\textbf{input} $m^{(0)} \in \R^d$, $\sigma^{(0)} > 0$}
\STATE{\textbf{parameters} $c_+ > 0$, $c_- < 0$}
\STATE {$t \leftarrow 0$}
\REPEAT
	\STATE {$\big(z^{(t)}\big) \sim \Normal(0, I)$}
	\STATE {$x^{(t)} \leftarrow m^{(t)} + \sigma^{(t)} \cdot z^{(t)}$}
	\IF {$f\big(x^{(t)}\big) \leq f\big(m^{(t)}\big)$}
		\STATE {$m^{(t+1)} \leftarrow x^{(t)}$}
		\STATE {$\sigma^{(t+1)} \leftarrow \sigma^{(t)} \cdot e^{c_+}$}
	\ELSE
		\STATE {$m^{(t+1)} \leftarrow m^{(t)}$}
		\STATE {$\sigma^{(t+1)} \leftarrow \sigma^{(t)} \cdot e^{c_-}$}
	\ENDIF
	\STATE {$t \leftarrow t + 1$}
\UNTIL{ \textit{stopping criterion is met} }
\end{algorithmic}
\end{algorithm}

\cite{jaegerskuepper2003analysis,jagerskupper2005rigorous,jaegerskuepper2006how,jaegerskuepper2006probabilistic}
analyzed the (1+1)-ES%
\footnote{Jägersküpper analyzed a different step size adaptation rule.
  However, it exhibits essentially the same dynamics as
  Algorithm~\ref{algorithm:oneplusoneES}.}
on the sphere function as well as on general convex quadratic functions.
His analysis ensures linear convergence with overwhelming probability,
i.e., with a probability of $1 - \exp \big( \Omega(d^\varepsilon) \big)$
for some $\varepsilon > 0$, where $d$ is the problem dimension. In other
words, the analysis is asymptotic in the sense $d \to \infty$, and for
fixed (finite) dimension $d \in \N$, no concrete value or bound is
attributed to this probability. A dimension-dependent convergence rate
of $\Theta(1/d)$ is obtained.

A related and more modern approach relying explicitly on drift analysis
was presented by \cite{akimoto2018drift}, showing linear convergence of
the algorithm on the sphere function, and providing an explicit,
non-asymptotic runtime bound for the first hitting time of a level set.

The analysis by \cite{auger2005convergence} is based on the stability of
the Markov chain defined by the normalized state $m / \sigma$, for a
$(1, \lambda)$-ES on the sphere function. Since the chain is shown to
converge to a stationary distribution and the problem is scale-invariant,
linear convergence or divergence is obtained, with full probability.
There exists sufficient empirical evidence for convergence, however,
this is not covered by the result.

A different approach to proving global convergence is to modify the
algorithm under consideration in a way that allows for an analysis with
well established techniques. This route was explored by
\cite{diouane2015globally}, where step size adaptation is subject to a
forcing function in order to guarantee a sufficient decrease condition,
akin to, e.g., the Wolfe conditions for inexact line search
\citep{wolfe1969convergence}.
This is a powerful approach since the resulting analysis is general in
terms of the algorithms (the same step size forcing mechanism can be
added to virtually all ES) and the objective functions (the function
must be bounded from below and Lipschitz near the limit point) at the
same time. The price is that the analysis does not apply to algorithms
regularly applied within the EC community, and that we do not obtain new
insights about the mechanisms of these algorithms. Furthermore, the
forcing function decays slowly, forcing a linearly convergent algorithm
into sub-linear convergence (but still much faster than random search).
From a more technical point of view the Lipschitz condition is
unfortunate since it is not preserved under monotonic transformations of
fitness values.
We improve on this approach by providing sufficient decrease of a
transformed objective function, which holds for all randomized elitist,
rank-based algorithms, and hence does not require a forcing function or
any other algorithmic changes.

The global convergence guarantee by \cite{akimoto2010theoretical} is
closest to the present paper. Also this analysis is extremely general in
the sense that it covers a broad range of problems and algorithms. The
objective function is assumed to be continuously differentiable, and the
only requirement for the algorithm is that it successfully diverges on a
linear function. This includes all state-of-the-art evolution strategies
and many more algorithms. Since continuously differentiable functions
are locally arbitrarily well approximated by linear functions (first
order Taylor polynomial), it is concluded that any limit point must be
stationary, since there the linear term vanishes and higher order terms
take over. This is an elegant and powerful result. Its main restriction
is that it applies only to continuously differentiable functions. This
is a huge class, but it can still be considered a relevant limitation
because on continuously differentiable problems ESs are in direct
competition with gradient-based methods, which are usually more
efficient if gradients are available.

For this reason, solving smooth and otherwise easy problems cannot be
the focus of evolution strategies. Therefore, in this paper we seek to
explore the most general class of problems that can be solved with an
evolution strategy. In other words, we aim to push the limits beyond the
well-understood cases, towards really difficult ones. Our goal is to
establish the largest possible class of problems can be be solved
reliably by an ES, and we also want to understand its limitations, i.e.,
which problems cannot be solved, and why. For this purpose, we focus on
the simplest such algorithm, namely the (1+1)-ES defined in
Algorithm~\ref{algorithm:oneplusoneES}. It turns out that the
limitations of the algorithm are closely tied to its success-based step
size adaptation mechanism. To capture this effect we introduce a novel
regularity condition ensuring proper function of success-based step-size
control. The new condition is arguably much weaker than continuous
differentiability, in a sense that will become clear as we discuss
examples and counter-examples.

From a bird's eye's perspective, our contributions are as follows:
\begin{enumerate}
\item
	we provide a general progress or decrease guarantee for rank-based
	elitist algorithms,
\item
	we show how general the (1+1)-ES is applicable, i.e., on which
	problems it will find a local optimum.
\end{enumerate}

The paper and the proofs are organized as follows. In the next section
we establish a progress guarantee for rank-based elitist algorithms.
This result is extremely general, and it is in no way tied to continuous
search spaces and the (1+1)-ES. Therefore it is stated in general
terms, in the expectation that it will prove useful for the analysis of
algorithms other than the (1+1)-ES. Its role in the global convergence
proof is to ensure a sufficient rate of optimization progress as long as
the step size is well adapted and the progress rate is bounded away from
zero.
In section~\ref{section:step-size-control} we discuss properties of the
(1+1)-ES and introduce the regularity condition. Based on this condition
we show that the step size returns infinitely often to a range where
non-trivial progress can be concluded from the decrease theorem.
Based on these achievements we establish a global convergence theorem in
section~\ref{section:global-convergence}, essentially stating that there
exists a sub-sequence of iterates converging to a critical point, the
exact notion of which is defined in
section~\ref{section:step-size-control}. We also establish a negative
result, showing that a non-optimal critical point results in premature
convergence with positive probability, which excludes global
convergence.
In section~\ref{section:case-studies} we apply the analysis to a variety
of settings and demonstrate their implications. We close with
conclusions and open questions.

%%%%%%%%%%%%%%%%%%%%%%%%%%%%%%%%%%%%%%%%%%%%%%%%%%%%%%%%%%%%
%%
\section{Optimization Progress of Rank-based Elitist Algorithms}
\label{section:decrease}

In this section, we establish a general theorem ensuring a certain rate
of optimization progress for randomized rank-based elitist algorithms.
We consider a general search space $X$. This space is equipped with a
$\sigma$-algebra and a reference measure denoted $\Lambda$. The usual
choice of the reference measure is the counting measure for discrete
spaces and the Lebesgue measure for continuous spaces. The objective
function $f : X \to \R$, to be minimized, is assumed to be measurable.
The parent selection and variation operations of the search algorithm
are also assumed to be measurable; indeed we assume that these operators
give rise to a distribution from which the offspring is sampled, and
this distribution has a density with respect to~$\Lambda$.

A rank-based optimization algorithm ignores the numerical fitness scores
($f$-values), and instead relies solely on pairwise comparisons,
resulting in exactly one of the relations $f(x) < f(x')$, $f(x) = f(x')$,
or $f(x) > f(x')$. This property renders it invariant to strictly
monotonically increasing (rank preserving) transformations of the
objective values. Therefore it ``perceives'' the objective function only
in terms of its level sets, not in terms of the actual function values.
For $f : X \to \R$ let
\begin{align*}
	L_f(y) & := \Big\{ x \in X \,\Big|\, f(x) = y \Big\} \\
	S_f^{<}(y) & := \Big\{ x \in X \,\Big|\, f(x) < y \Big\} \\
	S_f^{\leq}(y) & := \Big\{ x \in X \,\Big|\, f(x) \leq y \Big\}
\end{align*}
denote the level set of $f$, and the sub-level sets strictly below and
including level~$y \in \R$. For $m \in X$ we define the short notations
$L_f(m) := L_f\big(f(m)\big)$,
$S_f^{<}(m) := S_f^{<}\big(f(m)\big)$ and
$S_f^{\leq}(m) := S_f^{\leq}\big(f(m)\big)$.

Due to the assumption that the offspring generation distribution is
$\Lambda$-measurable, with full probability, the algorithm is invariant
to the values of the objective function restricted to zero sets (sets
$Z$ of measure zero, fulfilling $\Lambda(Z) = 0$).
The following definition captures these properties. It encodes the
``essential'' level set structure of an objective function.
\begin{definition} \label{definition:equivalence}
We call two measurable functions $f, g : X \to \R$ equivalent and write
\begin{align*}
	f \widehat\sim g
\end{align*}
if there exists a zero set $Z \subset X$ and a strictly monotonically
increasing function $\phi : f(X) \to g(X)$ such that
$g(x) = \phi\big(f(x)\big)$ for all $x \in X \setminus Z$. Here $f(X)$
and $g(X)$ denote the images of $f$ and $g$, respectively. We denote the
corresponding equivalence class in the set of measurable functions by
$[f] := \big\{g : X \to \R \,\big|\, g \widehat\sim f\big\}$.
\end{definition}
It follows immediately from the definition that the sub-level sets of
equivalent objective functions $f \widehat\sim g$ coincide outside a zero set.

In the next step we construct a canonical representative for each
equivalence class, which we can think of as a \emph{normal form} of an
objective function.
\begin{definition}
For $f : X \to \R$ we define the spatial suboptimality functions
\begin{align*}
	\f^{<} : X \to \R \cup \{\infty\},& \quad x \mapsto \Lambda \big( S_f^{<}(x) \big) \\
	\f^{\leq} : X \to \R \cup \{\infty\},& \quad x \mapsto \Lambda \big( S_f^{\leq}(x) \big),
\end{align*}
computing the volume of the success domain, i.e., the set of improving
points.
If $\f^{<}$ and $\f^{\leq}$ coincide then we drop
the upper index and simply denote the spatial suboptimality function
by~$\f$.
\end{definition}

\begin{figure}
\begin{center}
	\includegraphics[scale=1.25]{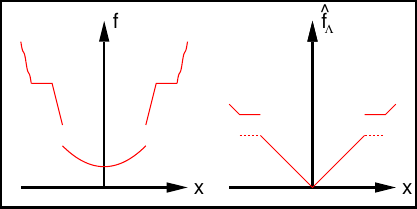}
	\caption{
		Left: objective function $f : \R \to \R$ with plateau and jump.
		Right: corresponding spatial suboptimality $\f^{<}$ (dotted) and
		$\f^{\leq}$ (solid).
		\label{figure:suboptimality1}
	}
\end{center}
\end{figure}

\begin{figure}
\begin{center}
	\includegraphics[scale=1.25]{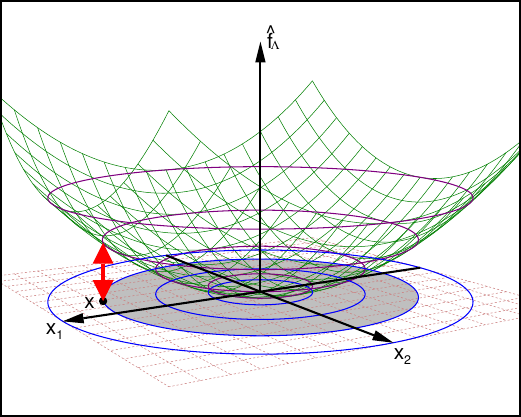}
	\caption{
		All relevant properties of the sphere function $f : \R^2 \to \R$
		for rank-based optimization are specified by its circular level
		sets, illustrated in blue on the domain (ground plane). The
		spatial sub-optimality of the point $x$ is the Lebesgue measure
		of the gray area, which coincides with the function value
		$\f(x)$ indicated by the bold red vertical arrow. In this
		example it holds $\f(x) = \pi \cdot \|x\|^2$, irrespective of
		the rank-preserving (and hence level-set preserving)
		transformation applied to $f$.
		\label{figure:suboptimality2}
	}
\end{center}
\end{figure}

The definition is illustrated with two examples in figures
\ref{figure:suboptimality1} and \ref{figure:suboptimality2}.
In the following, $m \in X$ will denote the elite (or parent) point, and
$m^{(t)}$ is the elite point in iteration $t \in \N$ of an iterative
algorithm, i.e., an evolutionary algorithm with $(1 + \lambda)$
selection.
For two very different reasons, namely
  1) to avoid divergence of the algorithm in the case of unbounded search spaces, and
  2) for simplicity of the technical arguments in the proofs,
we restrict ourselves to the case that the sub-level set
$S_{f}^{\leq}\big(m^{(0)}\big)$ of the initial iterate $m^{(0)}$ is
bounded and has finite spatial sub-optimality. For most reasonable
reference measures, boundedness implies finite spatial sub-optimality.
For $X = \R^d$ equipped with the Lebesgue measure this is equivalent
to the topological closure $\closure{S_f^{\leq}\big(m^{(0)}\big)}$ being
compact. The assumptions immediately imply that $S_{f}^{<}(y)$ and
$S_{f}^{\leq}(y)$ are bounded for all $y \leq f \big( m^{(0)} \big)$,
and that restricted to $S_{f}^{\leq}(m^{(0)})$ the functions
$\f^{<}$ and $\f^{\leq}$ take values in the bounded
range $\big[0, \f\big(m^{(0)}\big)\big]$. Since an elitist
algorithm never accepts points outside $S_{f}^{\leq}(m^{(0)})$, we will
from here on ignore the issue of infinite $\f$-values.%
\footnote{An alternative approach to avoiding infinite values is to
  apply a bounded reference measure with full support, e.g., a Gaussian
  on $\R^d$. In the absence of a uniform distribution on $X$, the price
  to pay for a bounded and everywhere positive reference measure is a
  non-uniform measure, which does not allow for a uniform, positive
  lower bound. The resulting technical complications seem to outweigh
  the slightly increased generality of the results.}

%The canonical representation in terms of spatial suboptimality is
%convenient for defining essential optima (minima) of~$f$, and actually
%of the whole class~$[f]$.
%%
%\begin{definition}
%Consider a measurable function $f : X \to \R$. The set
%${(\f^{\leq})}^{-1}(0)$ is called the set of essential global
%optima of~$f$. A point $x \in X$ is called an essential local optimum of
%$f$ if it is a local optimum of $\f^{\leq}$, i.e., if it is minimal
%restricted to an (open) neighborhood $N \subset X$ of~$x$.
%\end{definition}
%
In the continuous case, a plateau is a level set of positive Lebesgue
measure.
When defining a local optimum as the best point within an
open neighborhood, then an interior point of a plateau
is a local optimum, which may not always be intended. Anyway, when
analyzing the (1+1)-ES we will not handle plateaus and instead assume
that level sets of $f$ are zero sets. This also implies that $\f^{\leq}$
and $\f^{<}$ agree.

\begin{lemma}
\label{lemma:equivalence}
Let $f : X \to \R$ be measurable. If $\f^{\leq}(x)$ is finite
for all $x \in X$ then it holds
$\f^{\leq} \widehat\sim f \widehat\sim \f^{<}$.
\end{lemma}
The proof is found in the appendix.
We use $\f^{\leq}$ and $\f^{<}$ (or simply $\f$ if possible) as a
canonical representative of its equivalence class (if the function
values are finite, but see the discussion above). These functions have
the property
\begin{align*}
	\f^{\leq}(x) = \Lambda \big( S_{\f^{\leq}}(x) \big)
	\qquad
	\f^{<}(x) = \Lambda \big( S_{\f^{<}}(x) \big)
\end{align*}
i.e., $\f$ encodes the Lebesgue measure of its own sub-level
sets.
We will measure optimization progress in terms of $\f$-values.
Decreasing the spatial suboptimality $\f$ by $\delta > 0$
amounts to reducing the volume of better points by~$\delta$.

Due to the rank-based nature of the algorithms under study we cannot
expect to fulfill a sufficient decrease condition based on $f$-values.
This is because a functional gain $\Delta := f(x) - f(x') > 0$ achieved
by moving from $x$ to $x'$ can be reduced to an arbitrarily small or
large gain $\varphi(f(x)) - \varphi(f(x'))$, where $\varphi$ is strictly
monotonically increasing, and the class of transformations does not
allow to bound the difference uniformly, neither additively nor
multiplicatively.
Instead, the following theorem establishes a progress or decrease
guarantee measured in terms of the spatial suboptimality function~$\f$.
It gets around the problem of inconclusive values in objective space
(which, in case of single-objective optimization, is just the real line)
by considering a quantity in \emph{search space}, namely the reference
measure of the sub-level set.

The algorithm is randomized, hence the decrease follows a distribution.
The following definition captures properties of this distribution.
\begin{definition} \label{definition:notations}
Let $P$ denote a probability distribution on $X$ with a bounded density
with respect to $\Lambda$ and let $f : X \to \R$ be a measurable
objective function.
The quantity
\begin{align*}
	u := \sup \left\{ \left. \frac{P(A)}{\Lambda(A)} \,\right|\, A \subset X \text{ measurable with } \Lambda(A) > 0 \right\}
\end{align*}
is an upper bound on the density.
Consider a sample $x \sim P$. Define the functions
\begin{align*}
	r^{<}    : \R \to [0, 1], \qquad & z \mapsto \Pr \big( f(x) <    z \big) = P\left(S_f^{<   }(z)\right) \\
	r^{\leq} : \R \to [0, 1], \qquad & z \mapsto \Pr \big( f(x) \leq z \big) = P\left(S_f^{\leq}(z)\right)
\end{align*}
of probabilities of strict and weak improvements.
Furthermore, we define
$s : [0, 1] \to \R$ as a measurable inverse function fulfilling
$r^{<}\big(s(q)\big) \leq q \leq r^{\leq}\big(s(q)\big)$ for all
$q \in [0, 1]$.
We collect the discontinuities of $r^{<}$ and $r^{\leq}$
in the set $Z := \big\{z \in \R \,\big|\, r^{<}(z) < r^{\leq}(z) \big\}$
and define the sum
\begin{align*}
	\zeta := \sum_{z \in Z} \Big( r^{\leq}(z) - r^{<}(z) \Big)^2
\end{align*}
of squared improvement jumps.
\end{definition}
Note that $u$, $r^{<}$, $r^{\leq}$, $s$, $Z$, and $\zeta$ implicitly
depend on $\Lambda$, $P$, and $f$. This is not indicated explicitly in
order to avoid excessive clutter in the notation.

If the function $f$ is continuous with continuous domain $X$ and without
plateaus, then $r^<$ and $r^\leq$ coincide, we have $\zeta = 0$, and $s$
maps each probability $q \in [0, 1]$ to the corresponding unique
quantile of the distribution of $f(x)$ under $P$. However, if there
exists a plateau within the support of $P$ (a level set of positive
$P$-measure, i.e., if $X$ is discrete), then $\zeta$ is positive and on
$Z$ the function $s$ takes values anywhere between the lower quantile
$P(f(x) < z)$ and the upper quantile $P(f(x) \leq z)$. The exact value
does not matter, since the only use of $s$-values is as arguments to one
of the $r$-functions. Indeed, $r^<(s(q))$ and $r^\leq(s(q))$ ``round''
the probability $q$ down or up, respectively, to the closest value that
is attainable as the probability of sampling a sub-level set. The
freedom in the choice of $s$ can also be understood in the context of
figure~\ref{figure:suboptimality1}: if the point $z$ in the definitions
of $r^{<}$ and $r^{\leq}$ is located on the plateau, then $s(q)$ can be
the anywhere between the probability mass of the sub-level set excluding
and including the plateau.

With these definitions in place, the following theorem controls the
expected value as well as the quantiles of the decrease distribution.

\begin{theorem} \label{theorem:decrease}
Let $P$ denote a probability distribution on $X$ with a bounded density
with respect to $\Lambda$ and let $f : X \to \R$ be a measurable
objective function. We use the notation of the above definition.
Fix a reference point $m \in X$ and let $p := r^{<} \big( f(m) \big)$
denote the probability of strict improvement of a sample $x \sim P$ over
$m$. Then for each $q \in [0, p]$, the $q$-quantile of the
$\f^{<}$-decrease is bounded from below by $\frac{p - r^{<}\big(s(q)\big)}{u}$
and the $q$-quantile of the $\f^{\leq}$-decrease is bounded by
$\frac{p - r^{\leq}\big(s(q)\big)}{u} + \Lambda \big( L_f(m) \big)$, i.e.,
\begin{align*}
	& \Pr \left( \f^{<}(m) - \f^{<}(x) \geq \frac{p - r^{<}\big(s(q)\big)}{u} \right) \geq q, \\
	& \Pr \left( \f^{\leq}(m) - \f^{\leq}(x) \geq \frac{p - r^{\leq}\big(s(q)\big)}{u} + \Lambda \big( L_f(m) \big) \right) \geq q.
\end{align*}
The expected $\f^{<}$-decrease is bounded from below by
\begin{align*}
	\Expectation \Big[ \max \big\{ 0, \f^{<}(m) - \f^{<}(x) \big\} \Big] \geq \frac{p^2 + \zeta}{2u},
\end{align*}
and the expected $\f^{\leq}$-decrease is bounded from below by
\begin{align*}
	\Expectation \Big[ \max \big\{ 0, \f^{\leq}(m) - \f^{\leq}(x) \big\} \Big] \geq \frac{p^2 + \zeta}{2u} + \Lambda \big( L_f(m) \big).
\end{align*}
\end{theorem}
\begin{proof}
We start with the first two claims, which provide lower bounds on the
$q$-quantiles of probabilities of improvement by some margin $\delta
\geq 0$. The argument here is elementary: an $\f$-improvement of
$\delta$ from $m$ to $x$
%(which we can think of as parent~$m$ and offspring~$x$ in the (1+1)-ES)
means that the $\f$-sub-level set of $x$
is smaller than that of $m$ by $\Lambda$-mass $\delta$ (due to the
offspring $x$ improving upon its parent $m$). This corresponds to a
difference in $P$-mass of the same $\f$-sub-level sets of at most
$u \cdot \delta$, which will correspond to~$q$ in the following. Note
that the probabilities ($\Pr(\dots)$-notation) correspond to the same
distribution $P$ from which $x$ is sampled, and that $\f$-values and
$s$-values directly correspond to $\Lambda$-mass. The situation is
illustrated in figure~\ref{figure:decrease}.

To make the above argument precise we fix $q$ and define the $f$-level
$$ y_q := \inf\Bigg(\Big\{y \in \R \,\Big|\, P(S_f^{\leq}(y)) \geq q \Big\}\Bigg). $$
For $q = 0$ the first two statements are trivial. For $q > 0$ the
infimum is attained and it thus holds $P\big(S_f^{\leq}(y_q)\big) \geq q$.
We define three disjoint sets: $A := S_f^{<}(y_q)$, $B := L_f(y_q)$, and
$C := S_f^{<}(m) \setminus S_f^{\leq}(y_q)$. The nested sub-level sets
$S_f^{<}(y_q) = A$, $S_f^{\leq}(y_q) = A \cup B$, and
$S_f^{<}(m) = A \cup B \cup C$ are unions of these sets.
By the definitions of $p$ and $q$ the probability of the set $C$ is
upper bounded by $P(C) = P(S_f^<(m)) - P(S_f^{\leq}(y_q)) \leq p - q$,
and the probability of $A \cup B$ is lower bounded by
$P(A \cup B) \geq q$.

We will show that the event of interest for the first claim, namely
$\f^{<}(m) - \f^{<}(x) \geq \frac{p - r^{<}(s(q))}{u}$,
implies $x \in S_f^{\leq}(y_q) = A \cup B$. To this end we define the
$\f^{<}$-level $z_q^{<} := \f^{<}(m) - \frac{p - r^{<}(s(q))}{u}$ and
the set
$\Delta_q^{<} := S_{\f^{<}}^{<}(m) \setminus S_{\f^{<}}^{\leq}(z_q^{<})$.
We have
\begin{align*}
	\Lambda(\Delta_q^{<}) &= \underbrace{\Lambda \left( S_{\f^{<}}^<(m) \right)}_{= \f^{<}(m)} - \underbrace{\Lambda \left( S_{\f^{<}}^{\leq}(z_q^{<}) \right)}_{\geq \f^{<}(m) - \frac{p - r^{<}(s(q))}{u}} \leq \frac{p - r^{<}(s(q))}{u},
\end{align*}
and hence $P(\Delta_q^{<}) \leq p - r^{<}(s(q))$ by the definition of
$u$. Together with Lemma~\ref{lemma:equivalence} this implies
$\Delta_q^{<} \subset B \cup C$, and hence
$A \subset S_{\f^{<}}^{\leq}(z_q^{<})$. However, due to the
definition of $S^\leq$ (in contrast to $S^{<}$), the sub-level set $A$
being a subset of $S_{\f^{<}}^{\leq}(z_q^{<})$ implies that also the
level set $B$ is contained in $S_{\f^{<}}^{\leq}(z_q^{<})$. This shows
the first claim.

For the second claim we define the $\f^{\leq}$-level
$z_q^{\leq} := \f^{\leq}(m) - \frac{p - r^{\leq}(s(q))}{u} - \Lambda\big(L_f(m)\big)$
and the set
$\Delta_q^{\leq} := S_{\f^{\leq}}^{<}(m) \setminus S_{\f^{\leq}}^{\leq}(z_q)$,
and we note that it holds $\f^{\leq}(m) - \Lambda\big(L_f(m)\big) = \f^{<}(m)$.
Then, with an analogous argument as above we obtain
$P(\Delta_q^{<}) \leq p - r^{\leq}(s(q))$. In this case we immediately
arrive at $\Delta_q^{\leq} \subset C$ and hence at
$A \cup B \subset S_{\f^{\leq}}^{\leq}(z_q^{\leq})$, which shows the
second claim.

\begin{figure}
\begin{center}
	\includegraphics{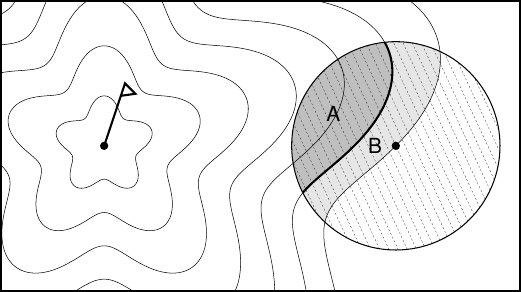}
	\caption{
		Illustration of the quantile decrease, here in the continuous
		case. The optimum is marked with a flag. In this example, the
		level lines of the objective function $f$ are star-shaped. The
		circle with the dashed shading on the right indicates the
		sampling distribution, which has ball-shaped support in this
		case. The probability of the area $A \cup B$ is the value
		$p = P(A \cup B)$, and $q = P(A)$ is the probability of the
		event of interest, corresponding to a significant improvement.
		The area $\Lambda(B)$ is a lower bound on the improvement in
		terms of $\f$. It is lower bounded by
		$\frac{P(B)}{u} = \frac{p - q}{u}$.
		The (bold) level line separating $A$ and $B$ belongs to $A$, and
		not to $B$. Therefore, if this set has positive measure, then we
		can only guarantee $q \leq P(A)$ (in contrast to equality), and
		the lower bound becomes $\frac{p - r^{<}(s(q))}{u} \leq \Lambda(B)$.
		\label{figure:decrease}
	}
\end{center}
\end{figure}

Let $Q$ denote the quantile function (the generalized inverse of the
cdf) of the $\f^{<}$-improvement
$\max \big\{ 0, \f^{<}(m) - \f^{<}(x) \big\}$. Then the expectation is
lower bounded by
\begin{align*}
	\Expectation \Big[ \max \big\{ 0, \f^{<}(m) - \f^{<}(x) \big\} \Big]
		=    & \int_0^1 Q(q) \,dq \\
		\geq & \int_0^p \frac{p - r^{<}\big(s(q)\big)}{u} \,dq \\
		=    & \int_0^p \frac{p - q}{u} \,dq + \int_0^p \frac{q - r^{<}\big(s(q)\big)}{u} \,dq \\
		=    & \int_0^p \frac{p - q}{u} \,dq + \sum_{z \in Z} \int_{r^{<}(z)}^{r^{\leq}(z)} \frac{q - r^{<}(z)}{u} \,dq \\
		=    & \frac{p^2}{2u} + \sum_{z \in Z} \frac{\big(r^{\leq}(z) - r^{<}(z)\big)^2}{2u} \\
		=    & \frac{p^2 + \zeta}{2u}.
\end{align*}
The proof of the expected $\f^{\leq}$ improvement is analogous. The
additional term $\Lambda \big( L_f(m) \big)$ again comes from
$\f^{\leq}(m) = \f^{<}(m) + \Lambda\big(L_f(m)\big)$.
%
%By construction it holds $P\big( L_f(z) \big) = r^{\leq}(z) - r^{<}(z)$
%for all $z \in \R$. This accounts for the differences between the first
%and second and between the third and fourth statement; these differ only
%by including or excluding the level set $L_f(z)$ in the events under
%consideration.
\end{proof}
In a $(1+\lambda)$ evolutionary algorithm, the reference point $m$ is
usually the current best iterate (parent), and the distribution $P$ is
the search distribution, from which the offspring are sampled. Our main
application is the (1+1)-ES, where $x$ corresponds to the offspring
point sampled from a Gaussian centered on~$m$.

Due to the term $\Lambda\big(L_f(m)\big)$ in the decrease of
$\f^{\leq}$, the theorem covers the fitness-level method
\citep{droste2002analysis,wegener2003methods}. However, in particular
for search distributions spreading their probability mass over many
level sets, the theorem is considerably stronger.

In the continuous case, in the absence of plateaus, the statement can be
simplified considerably:
\begin{corollary} \label{corollary:plateau-free-decrease}
Under the assumptions and with the notation of
definition~\ref{definition:notations} and
theorem~\ref{theorem:decrease} we assume in addition that all
level sets of $f$ have measure zero. Then for each
$q \in [0, p]$, the $q$-quantile of the $\f$-decrease is
bounded from below by
\begin{align*}
	\Pr \left( \f(m) - \f(x) \geq \frac{p - q}{u} \right) \geq q,
\end{align*}
and the expected $\f$-decrease is bounded from below by
\begin{align*}
	\Expectation \Big[ \max \big\{ 0, \f(m) - \f(x) \big\} \Big] \geq \frac{p^2}{2u}.
\end{align*}
\end{corollary}
The following corollary is a broken down version for Gaussian search
distributions $\Normal(m, C)$ with mean $m$ and covariance matrix~$C$,
which has the density
\begin{align*}
	\varphi(x) = \frac{1}{(2\pi)^{d/2} \sqrt{\det(C)}} \exp \left( -\frac{1}{2} (x-m)^T C^{-1} (x-m) \right).
\end{align*}
\begin{corollary} \label{corollary:gaussian-decrease}
Consider the search space $\R^d$ and the Lebesgue measure $\Lambda$.
Let $f : \R^d \to \R$ denote a measurable objective function with level
sets of measure zero. Consider a normally distributed sample
$x \sim \Normal(m, C)$. Under the assumptions and with the notation of
definition~\ref{definition:notations} and
theorem~\ref{theorem:decrease}, for each $q \in [0, p]$, the
$q$-quantile of the $\f$-decrease is bounded from below by
\begin{align*}
	\Pr \Big( \f(m) - \f(x) \geq (2\pi)^{d/2} \cdot \sqrt{\det(C)} \cdot (p - q) \Big) \geq q,
\end{align*}
and the expected $\f$-decrease is bounded from below by
\begin{align*}
	\Expectation\Big[ \max \big\{ 0, \f(m) - \f(x) \big\} \Big] \geq (2\pi)^{d/2} \cdot \sqrt{\det(C)} \cdot \frac{p^2}{2}.
\end{align*}
\end{corollary}
An isotropic distribution with component-wise standard deviation (step
size) $\sigma > 0$ has covariance matrix $C = \sigma^2 I$, where
$I \in \R^{d \times d}$ is the identity matrix; hence we have
$\sqrt{\det(C)} = \sigma^d$. In the context of continuous search spaces,
\cite{jaegerskuepper2003analysis} refers to $\f$-progress as
``spatial gain''. He analyzes in detail the gain distribution of an
isotropic search distribution on the sphere model. This result is much
less general than the previous corollary, since we can deal with
\emph{arbitrary} objective functions, which are characterized (locally)
only by a single number, the success probability. For the special case
of a Gaussian mutation and the sphere function, Jägersküpper's
computation of the spatial gain is more exact, since it is tightly
tailored to the geometry of the case, in contrast to being based on a
general bound. We lose only a multiplicative factor of the gain, which
does not impact our analysis significantly. However, it should be noted
that in the problem analyzed by Jägersküpper, the factor grows with the
problem dimension~$d$.
The spatial gain is closely connected to the notion of a progress rate
\citep{rechenberg:1973}, in particular if the gain is lower bounded by a
fixed fraction of the suboptimality. For a fixed objective function like
the sphere model $f(x) = \|x\|^2$ it is easy to relate functional
suboptimality $f(x) - f^*$ to spatial suboptimality $\f(x)$.

The above statements apply immediately to evolutionary algorithms with
(1+1) selection. Generalizing them to the best out of $\lambda$ samples
is rather straightforward (based on well-known statements on the
distribution of the minimum of $\lambda$ i.i.d.\ samples), resulting in
bounds for the one-step behavior of elitist algorithms with
$(1 + \lambda)$ selection. This is not done here because we are not
primarily interested in the scaling with $\lambda$, which is usually
non-essential for the question whether or not a $(1 + \lambda)$
algorithm converges to a (local) optimum.

%%%%%%%%%%%%%%%%%%%%%%%%%%%%%%%%%%%%%%%%%%%%%%%%%%%%%%%%%%%%
%%
\section{Success-bases Step Size Control in the (1+1)-ES}
\label{section:step-size-control}

In this section we discuss properties of the (1+1)-ES algorithm and
provide an analysis of its success-based step size adaptation rule that
will allow us to derive global convergence theorems. To this end we
introduce a non-standard regularity property.

From here on, we consider the search space $\R^d$, equipped with the
standard Borel $\sigma$-algebra, and $\Lambda$ denotes the Lebesgue
measure. Of course, all results from the previous section apply, with
$X = \R^d$.

In each iteration $t \in \N$, the state of the (1+1)-ES is given by
$(m^{(t)}, \sigma^{(t)}) \in \R^d \times \R^+$. It samples one candidate
offspring from the isotropic normal distribution
$x^{(t)} \sim \Normal(m^{(t)}, \big(\sigma^{(t)})^2 I\big)$. The parent
is replaced by successful offspring, meaning that the offspring must
perform at least as good as the parent.

The goal of success-based step size adaptation is to maintain a stable
distribution of the success rate, for example concentrated around $1/5$.
This can be achieved with a number of different mechanisms. Here we
consider the maybe simplest such mechanism, namely immediate adaptation
based on ``success'' or ``failure'' of each sample. Pseudocode for the
full algorithm is provided in Algorithm~\ref{algorithm:oneplusoneES}.

The constants $c_- < 0$ and $c_+ > 0$ in
Algorithm~\ref{algorithm:oneplusoneES} control the change of
$\log(\sigma)$ in case of failure and success, respectively. They are
parameters of the method. For $c_+ + 4 \cdot c_- = 0$ we obtain an
implementation of Rechenberg's classic $1/5$-rule
\citep{rechenberg:1973}. We call $\tau = \frac{c_-}{c_- - c_+}$ the
target success probability of the algorithm, which is always assumed to
be strictly less than $1/2$. This is equivalent to $c_+ > -c_-$. A
reasonable parameter setting is
$c_-, c_+ \in \Omega\left(\frac{1}{d}\right)$.

Two properties of the algorithm are central for our analysis: it is
rank-based and it performs elitist selection, ensuring that the
best-so-far solution is never lost and the sequence $f(m^{(t)})$ is
monotonically decreasing.

Since step-size control depends crucially on the concept of a fixed rate
of successful offspring, we define the success probability of the
algorithm, which is the probability of a sampled point outperforming the
parent in the search distribution center.
\begin{definition}
For a measurable function $f : \R^d \to \R$, we define the \emph{success
probability} functions
\begin{align*}
	p_f^{<} : \R^d \times \R^+ \to [0, 1] , \quad
		(m, \sigma) \mapsto & \Pr\Big(f(x) <    f(m) \,\Big|\, x \in \Normal(m, \sigma^2 I)\Big) \\
						= & \int_{S_f^{<   }(m)} \frac{1}{(2\pi)^{d/2}\sigma^d} \cdot \exp \left( -\frac{\|x - m\|^2}{2 \sigma^2} \right) \,dx, \\
	p_f^{\leq} : \R^d \times \R^+ \to [0, 1] , \quad
		(m, \sigma) \mapsto & \Pr\Big(f(x) \leq f(m) \,\Big|\, x \in \Normal(m, \sigma^2 I)\Big) \\
						= & \int_{S_f^{\leq}(m)} \frac{1}{(2\pi)^{d/2}\sigma^d} \cdot \exp \left( -\frac{\|x - m\|^2}{2 \sigma^2} \right) \,dx.
\end{align*}
\end{definition}
The function $p_f^{\leq}$ computes the probability of sampling a point at
least as good as $m$, while $p_f^{<}$ computes the probability of sampling
a strictly better point. If $p_f^{<}$ and $p_f^{\leq}$ coincide (i.e., if there
are no plateaus), then we write~$p_f$.
A nice property of the success probability is that it does not drop
too quickly when increasing the step size:
\begin{lemma} \label{lemma:increasing-step-size}
For all $m \in \R^d$, $\sigma > 0$ and $a \geq 1$ it holds
\begin{align*}
	p_f^{<}   (m, a \cdot \sigma) &\geq \frac{1}{a^d} \cdot p_f^{<}   (m, \sigma), \\
	p_f^{\leq}(m, a \cdot \sigma) &\geq \frac{1}{a^d} \cdot p_f^{\leq}(m, \sigma).
\end{align*}
\end{lemma}
The proof is found in the appendix; this is the case for a number of
technical lemmas in this section.
The next step is to define a plausible range for the step size.
\begin{definition}
\label{definition:xi-eta}
For $p \in [0, 1]$ and a measurable function $f : \R^d \to \R$, we
define upper and lower bounds
\begin{align*}
	 \xi_p^f(m) := \inf & \Big\{ \sigma \in \R^+ \,\Big|\, p_f^{<}(m, \sigma) \leq p \Big\} \\
	\eta_p^f(m) := \sup & \Big\{ \sigma \in \R^+ \,\Big|\, p_f^{\leq}(m, \sigma) \geq p \Big\}
\end{align*}
on the step size guaranteeing lower and upper bounds on the probability
of improvement.
\end{definition}
We think of $\xi_p^f(m)$ with $p > \tau$ as a ``too small'' step size
at $m$. Similarly, for $p < \tau$, $\eta_p^f(m)$ is a ``too large''
step size at $m$. Assume that the two values of $p$ are chosen
so that a sufficiently wide range of ``well-adapted'' step sizes exists
in between the ``too small'' and ``too large'' ones.
We aim to establish that if the step size is outside this range then
step size adaptation will push it back into the range. The main
complication is that the range for $\sigma$ depends on the point~$m$.

The following lemma establishes a gap between lower and upper step size
bound, i.e., a lower bound on the size of the step size range.
\begin{lemma} \label{lemma:gap}
For $0 \leq p_H \leq p_T \leq 1$ it holds
$\sqrt[d]{p_H} \cdot \xi_{p_T}^f(x) \leq \sqrt[d]{p_T} \cdot \eta_{p_H}^f(x)$
for all $x \in \R^d$.
\end{lemma}

The following definition is central. It captures the ability of the
(1+1)-ES to recover from a state with a far too small step size. This
property is needed to avoid premature convergence.
\begin{definition} \label{definition:regularity}
For $p > 0$, a function $f : \R^d \to \R$ is called $p$-improvable in
$x \in \R^d$ if $\xi_p^f(x)$ is positive. The function is called
$p$-improvable on $Y \subset \R^d$ if $\xi_p^f\big|_Y$ (the function
$\xi_p^f$ restricted to $Y$) is lower bounded by a positive, lower
semi-continuous function $\tilde \xi_p^f : Y \to (0, 1]$. A point
$x \in \R^d$ is called $p$-critical if it is not $p$-improvable for any
$p > 0$.
\end{definition}

\begin{figure}
\begin{center}
	\includegraphics{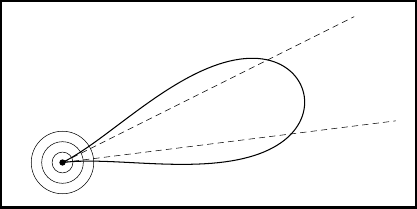}
	\caption{
		Illustration of a contour line with a kink opening up in an
		angle indicated by the dashed lines. The circles are iso-density
		lines on the isotropic Gaussian search distribution centered on
		the kink.
		\label{figure:angle}
	}
\end{center}
\end{figure}

The property of $p$-improvability is a non-standard regularity
condition. The concept applies to measurable functions, hence we do not
need to restrict ourselves to smooth or continuous objectives. On the
one hand side, the property excludes many measurable and even a few
smooth functions. On the other hand, it is far less restrictive than
continuity and smoothness, since it allows the objective function to
jump and the level sets to have kinks. Intuitively, in the
two-dimensional case illustrated in figure~\ref{figure:angle}, if for
each point the sub-level set opens up in an angle of more than $2 \pi
p$, then the function is $p$-improvable. This is the case for many
discontinuous functions, however, not for all smooth ones. The degree
three polynomial $f(x_1, x_2) = x_1^3 + x_2^2$ can serve as a counter
example, since the point $(0, 0)$ is $p$-critical, with its level set
forming a cuspidal cubic, see figure~\ref{figure:x3y2}
in section~\ref{section:cubic-saddle}.
Local optima are always $p$-critical, but many critical points of smooth
functions are not (see below). The above example demonstrates that some
saddle points share this property, however, if $x$ is $p$-critical but
not locally optimal then $p_f^{<}(x, \sigma) > 0$ for all $\sigma > 0$.
This means that such a point can be improved with positive probability
for each choice of the step size, but in the limit $\sigma \to 0$ the
probability of improvement tends to zero.

We should stress the difference between point-wise $p$-improvability,
which simply demands that $\xi_p^f$ is positive, and set-wise
$p$-improvability, which in addition demands that $\xi_p^f$ is lower
bounded by a lower semi-continuous positive function. The latter
property ensures the existence of a positive lower bound for $\xi_p^f$
on a compact set. In this sense, set-wise $p$-improvability is uniform
on compact sets. In sections \ref{section:jump} and
\ref{section:fat-cantor} we will see examples where this makes a
decisive difference.

Intuitively, the value of $p$ of a $p$-improvable function is critical:
if it is below $\tau$ then the algorithm may be endangered to
systematically decrease its step size while it should better do the
contrary.
%However, with a covariance matrix adaptation (CMA) mechanism
%in place, a small angle representing the region of improvement can be
%opened up, which can bring the success probability arbitrarily close to
%$1/2$. Hence, for many (but not for all) cases this implies that with a
%CMA-style mechanism in place the exact value of $p$ does not matter, as
%long as a point is not $p$-critical. Covariance matrix adaptation is
%outside the scope of this paper, therefore our guarantees apply to
%$p$-improvable functions for some $p > \tau$.

The next lemma establishes that smooth functions are $p$-improvable in
all regular points, and also in most saddle points.
\begin{lemma} \label{lemma:differentiable-improvable}
Let $f : \R^d \to \R$ be continuously differentiable.
\begin{compactenum}
\item
	For a regular point $x \in \R^d$, $f$ is $p$-improvable in $x$ for
	all $p < \frac{1}{2}$.
\item
	Let $Y$ denote the set of all regular points of $f$, then $f$ is
	$p$-improvable on $Y$, for all $p < \frac{1}{2}$.
\item
	Let $x \in \R^d$ denote a critical point of $f$, let $f$ be twice
	continuously differentiable in a neighborhood of $x$, and let
	$H = \nabla^2 f(x)$ denote the Hessian matrix. If $H$ has at least
	one negative eigen value, then $x$ is not $p$-critical.
\end{compactenum}
\end{lemma}
Similarly, we need to ensure that the step size does not diverge to
$\infty$. This is easy, since the spatial suboptimality is finite:
\begin{lemma} \label{lemma:upper-bound}
Consider the state $(m^{(t)}, \sigma^{(t)})$ of the (1+1)-ES.
For each $p \in (0, 1)$, if
\begin{align*}
	\sigma^{(t)} \geq \sqrt[d]{\frac{\f(m^{(t)})}{p \cdot (2\pi)^{d/2}}}
\end{align*}
then $p_f^{<}(m^{(t)}, \sigma^{(t)}) \leq p$.
\end{lemma}
In other words, a too large step size is very likely to produce
unsuccessful offspring. The probability of success decays quickly with
growing step size, since the step size bound grows slowly in the form
$\Theta(p^{-1/d})$ as the success probability $p$ decays to zero.
Applying the above inequality to $p < \tau$ implies that for large
enough step size $\sigma^{(t)}$, the expected change
$\Expectation[\log(\sigma^{(t+1)}) - \log(\sigma^{(t)})]$
in the (1+1)-ES (algorithm~\ref{algorithm:oneplusoneES})
is negative.

The following lemma is elementary. It is used multiple times in proofs,
with the interpretation of the event ``1'' meaning that a statement
holds true. It has a similar role as drift theorems in an analysis of
the expected or high-probability behavior
\citep{lehre2013general,lengler2016drift,akimoto2018drift}, however,
here we aim for almost sure results.
\begin{lemma} \label{lemma:subsequence}
Let $X^{(t)} \in \{0, 1\}$ denote a sequence of independent binary
random variables. If there exists a uniform lower bound
$\Pr(X^{(t)} = 1) \geq p > 0$, then almost surely there exists an
infinite sub-sequence $(t_k)_{k \in \N}$ so that $X^{(t_k)} = 1$ for all
$k \in \N$.
\end{lemma}
In applications of the lemma, the events of interest are not necessarily
independent, however, they can be ``made independent'' by considering a
sequence of independent events that imply the events of interest. In our
applications this is the case if the events of actual interest hold with
probability of at least $p$; then an i.i.d.\ sequence of Bernoulli
events implying corresponding sub-events with probability of exactly $p$
does the job. In other words, we will have a sequence $\tilde X^{(t)}$
of independent events, where $\tilde X^{(t)} = 1$ implies $X^{(t)} = 1$.
The above lemma is then applied to $\tilde X^{(t)}$, which trivially
yields the same statement for $X^{(t)}$. We imply this construction in
all applications of the lemma.

The following lemma establishes, under a number of technical conditions,
that the step size control rule succeeds in keeping the step size
stable. If the prerequisites are fulfilled then the result yields an
impossible fact, namely that the overall reduction of the spatial
suboptimality is unbounded. So the lemma is designed with proofs by
contradiction in mind.
\begin{lemma} \label{lemma:sigma-in-range}
Let $\big( m^{(t)}, \sigma^{(t)} \big)$ denote the sequence of states of
the (1+1)-ES on a measurable objective function $f : \R^d \to \R$. Let
$p_T, p_H \in (0, 1)$ denote probabilities
fulfilling $p_H < \tau < p_T$ and $\frac{p_H}{p_T} \leq e^{d \cdot c_-}$,
and assume the existence of constants $0 < b_T < b_H$ such that
\begin{align*}
	b_T \leq \xi_{p_T}^f(m^{(t)}) \quad \text{and} \quad e^{c_+} \cdot \eta_{p_H}^f(m^{(t)}) \leq b_H
\end{align*}
for all $t \in \N$.
Then, with full probability, there exists an infinite sub-sequence
$(t_k)_{k \in \N}$ of iterations fulfilling
\begin{align}
	\sigma^{(t_k)} \in \Big[ \xi_{p_T}^f \big( m^{(t_k)} \big), \eta_{p_H}^f \big( m^{(t_k)} \big) \Big]
	\label{eq:sigma-in-range}
\end{align}
for all $k \in \N$.
\end{lemma}

Equation~\eqref{eq:sigma-in-range} is a rather weak condition demanding
that step-size adaptation works as desired. However, the requirement of
a uniform lower bound $b_T$ on the step size together with
theorem~\ref{theorem:decrease} implies that the (1+1)-ES
would make infinite $\f$-progress in expectation. This is of course
impossible if $\f(m^{(0)})$ is finite, since $\f$ is by definition
non-negative. Therefore the lemma does not describe a typical situation
observed when running the (1+1)-ES, but quite in contrast, an impossible
situation that need to be excluded in the proof of the main result in
the next section.

%%%%%%%%%%%%%%%%%%%%%%%%%%%%%%%%%%%%%%%%%%%%%%%%%%%%%%%%%%%%
%%
\section{Global Convergence}
\label{section:global-convergence}

In this section we establish our main result. The theorem ensures the
existence of a limit point of the sequence $m^{(t)}$ in a subset of
desirable locations. In many cases this amounts to convergence of the
algorithm to a (local) optimum.

\begin{theorem} \label{theorem:global-convergence}
Consider a measurable objective function $f : \R^d \to \R$ with level
sets of measure zero. Assume that
$K_0 := \closure{S_{f}^{\leq}\big(m^{(0)}\big)}$ is compact, and let
$K_1 \subset K_0$ denote a closed subset. If $f$ is $p$-improvable on
$K_0 \setminus K_1$ for some $p > \tau$,
then the sequence $\big(m^{(t)}\big)_{t \in \N}$ has a limit point
in~$K_1$.
\end{theorem}
\begin{proof}
Lemma~\ref{lemma:upper-bound} ensures the existence of
$0 < p_H < e^{-d \cdot c_-} \cdot \tau$ and
\begin{align*}
	b_H := \sqrt[d]{\frac{\f \big( m^{(0)} \big)}{p_H \cdot (2\pi)^{d/2}}}
\end{align*}
such that it holds $\eta_{p_H}^f(x) \leq b_H$ uniformly for all
$x \in K_0$. In particular, $b_H$ is a uniform upper bound
on~$\eta_{p_H}^f$.

Let $B(x, r)$ denote the open ball of radius $r > 0$ around $x \in \R^d$
and define the compact set
\begin{align*}
	K(r) := K_0 \mathbin{\Big\backslash} \bigcup_{x \in K_1} B(x, r).
\end{align*}
It holds $K(r) \subset K_0 \setminus K_1$ and
$\bigcup_{r > 0} K(r) = K_0 \setminus K_1$; hence $K(r)$ is a compact
exhaustion of $K_0 \setminus K_1$.

Fix $r > 0$, and assume for the sake of contradiction that all points
$m^{(t)}$, $t > t_0$, are contained in $K(r)$. We set $p_T := p$. Let
$\tilde \xi_{p_T}^f$ denote the positive lower semi-continuous lower
bound on $\xi_{p_T}^f$, which is guaranteed to exist due to the
$p$-improvability of $f$. We define
\begin{align*}
	b_T := \min\Big\{\tilde \xi_{p_T}^f(m) \,\Big|\, m \in K(r)\Big\} > 0
\end{align*}
and apply lemma~\ref{lemma:sigma-in-range} to obtain an infinite
sub-sequence of states with step size lower bounded by
$\sigma^{(t)} \geq b_T > 0$. According to
lemma~\ref{lemma:increasing-step-size}, the success probability is lower
bounded by
$p_f\big(m^{(t)}, \sigma^{(t)}\big) \geq p_I := (b_T/b_H)^d \cdot p_T > 0$
for all $m \in K(r)$ and $\sigma \in [b_T, b_H]$.

Corollary~\ref{corollary:gaussian-decrease} ensures that in each such
state the probability to decrease the $\f$-value by at least
$(2\pi)^{d/2} \cdot b_T^d \cdot p_I/2$ is lower bounded by $p_I/2 > 0$.
We apply lemma~\ref{lemma:subsequence} with the following construction.
For each state $(m, \sigma)$ we pick a set $E(m, \sigma) \subset \R^d$
of probability mass $p_I/2$ improving on $\f(m)$ by at least
$(2\pi)^{d/2} \cdot b_T^d \cdot p_I/2$. Then we model the sampling
procedure of the (1+1)-ES in iteration $t$ as a two-stage process: first
we draw a binary variable $\tilde X^{(t)} \in \{0, 1\}$ with
$\Pr(\tilde X{(t)} = 1) = p_I/2$, and then we draw $x^{(t)}$ from a
Gaussian restricted to $E(m^{(t-1)}, \sigma^{(t-1)})$ if
$\tilde X{(t)} = 1$, and restricted to the complement otherwise. The
variables $\tilde X^{(t)}$ are independent, by construction.

Then lemma~\ref{lemma:subsequence} implies that the overall
$\f$-decrease is almost surely infinite, which contradicts the fact
that $\f(m^{(0)})$ is finite and $\f$ is lower bounded by zero. Hence,
the sequence $m^{(t)}$ leaves $K(r)$ after finitely many steps, almost
surely. For $r = 1/n$, let $t_n$ denote an iteration fulfilling
$m^{(t_n)} \not\in K(r)$. The sequence $\big(m^{(t_n)}\big)_{n \in \N}$
does not have a limit point in $K_0 \setminus K_1$ (since that point
would be contained in $K(r)$ for some $r > 0$), however, due to the
Bolzano-Weierstra{\ss} theorem it has at least one limit point in $K_0$,
which must therefore be located in~$K_1$.
\end{proof}

The above theorem is of primary interest if $K_1$ is the set of (local)
minima of $f$, or at least the set of critical or $p$-critical points.
Due to the prerequisites of the theorem we always have
\begin{align*}
	\closure{\Big\{ x \in K_0 \,\Big|\, x \text{ is $p$-critical} \Big\}} \subset K_1,
\end{align*}
i.e., $p$-critical points are candidate limit points.

In accordance with \cite{akimoto2010theoretical}, the following
corollary establishes convergence to a critical point for continuously
differentiable functions.
\begin{corollary} \label{corollary:differentiable}
Let $f : \R^d \to \R$ be a continuously differentiable function with
level sets of measure zero. Assume that
$K_0 = \closure{S_{f}^{\leq}\big(m^{(0)}\big)}$ is compact. Then the
sequence $\big(m^{(t)}\big)_{t \in \N}$ has a critical limit point.
\end{corollary}
\begin{proof}
Define $K_1 := \{x \in K_0 \,|\, \nabla f(x) = 0 \}$ as the set of
critical points. This set is compact.
Lemma~\ref{lemma:differentiable-improvable} ensures that $f$ is
$p$-improvable on $K_0 \setminus K_1$ for all $p < 1/2$.
Then the claim follows immediately from
theorem~\ref{theorem:global-convergence}.
\end{proof}

Technically the above statements do not apply to problems with unbounded
sub-level sets. However, due to the fast decay of the tails of Gaussian
search distributions we can often approximate these problems by
changing the function ``very far away'' from the initial search
distribution, in order to make the sub-level sets bounded. We may then
even apply the theorem with empty $K_1$, which implies that after a
while the approximation becomes insufficient since the algorithm
diverges. In this sense we can conclude divergence, e.g., on a linear
function. We will use this argument several times in the next section,
mainly to avoid unnecessary technical complications when defining saddle
points and ridge functions.

We may ask whether $p$-improvability for $p > \tau$ is not only a
sufficient but also a necessary condition for global convergence. This
turns out to be wrong. The quadratic saddle point case discussed below
in section~\ref{section:quadratic-saddle} is a counter example, where
the algorithm diverges reliably even if the success probability is far
smaller than $\tau$. In contrast, the ridge of $p$-critical saddle
points analyzed in section~\ref{section:cubic-saddle} results in
premature convergence, despite the fact that the critical points form a
zero set, and this can even happen for a ridge of $p$-improvable points
with $p < \tau$, see section~\ref{section:linear-ridge}. Drift analysis
is a promising tool for handling all of these cases. Here we provide a
rather simple result, which still suffices for many interesting cases.
A related analysis for a non-elitist ES was carried out by
\cite{beyer2006self}.
\begin{theorem} \label{theorem:premature-convergence}
Consider a measurable objective function $f : \R^d \to \R$ with level
sets of measure zero. Let $m \in \R^d$ be a $p$-critical point. If the
success probability decays sufficiently quickly, i.e., if
\begin{align*}
	\sum_{k=0}^{\infty} p_f^{\leq}(m, e^{k \cdot c_-}) < \infty
\end{align*}
then for each given $p < 1$ there exists an initial condition such that
the (1+1)-ES converges to $m$ with probability of at least~$p$.
\end{theorem}
\begin{proof}
Define the zero sequence
$S_K := \sum_{k=K}^{\infty} p_f^{\leq}(m, e^{k \cdot c_-})$. For given
$p < 1$, there exists a $K_0$ such that $S_{K_0} < 1 - p$. By definition,
the probability of never sampling a successful offspring when starting
the algorithm in the initial state $m^{(0)} = m$,
$\sigma^{(0)} = e^{K_0 \cdot c_-}$ is given by $S_{K_0}$; in this case
we have $m^{(t)} = m$ for all $t \in \N$.
\end{proof}
The above theorem precludes global convergence to a (local) optimum with
full probability in the presence of a suitable non-optimal $p$-critical
point.

%%%%%%%%%%%%%%%%%%%%%%%%%%%%%%%%%%%%%%%%%%%%%%%%%%%%%%%%%%%%
%%
\section{Case Studies}
\label{section:case-studies}

In this section we analyze various example problems with very different
characteristics, by applying the above convergence analysis. We
characterize the optimization behavior of the (1+1)-ES, giving either
positive or negative results in terms of global convergence. We start
with smooth functions and then turn to less regular cases of non-smooth
and discontinuous functions. On the one hand side, we show that the
theorem is applicable to interesting and non-trivial cases; on the other
hand we explore its limits.

\subsection{The 2-D Rosenbrock Function}

The two-dimensional Rosenbrock function is given by
\begin{align*}
	f(x_1, x_2) := 100 (x_1^2 - x_2)^2 + (x_1 - 1)^2.
\end{align*}
This is a degree four polynomial. The function is unimodal (has a single
local minimum), but not convex. Moreover, it does not have critical
points other than the global optimum $x^* = (1, 1)$. The function is
illustrated in figure~\ref{figure:rosenbrock}.

\begin{figure}
\begin{center}
	\includegraphics{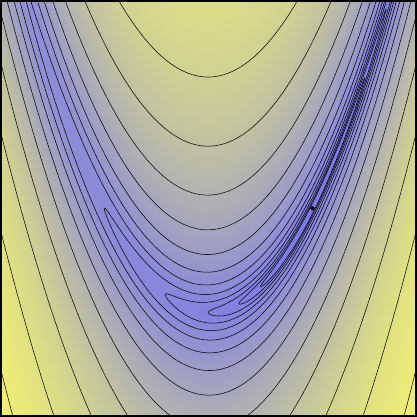}
	\caption{
		The 2-D Rosenbrock function in the range $[-2, 2] \times [-1, 3]$.
		\label{figure:rosenbrock}
	}
\end{center}
\end{figure}

The Rosenbrock function is a popular test problem since it requires a
diverse set of optimization behaviors: the algorithm must descend into a
parabolic valley, follow the valley while adapting to its curved shape,
and finally converge into the global optimum, which is a smooth optimum
with non-trivial (but still moderate) conditioning.

Corollary~\ref{corollary:differentiable} immediately implies convergence
of the (1+1)-ES into the global optimum. It does not say anything about
the speed of convergence, however, \cite{jaegerskuepper2006how}
established linear convergence in the last phase with overwhelming
probability (however, using a different step size adaptation rule).

Taken together, these results give a rather complete picture of the
optimization process: irrespective of the initial state we know that the
algorithm manages to locate the global optimum without getting stuck on
the way. Once the objective function starts to look quadratic in good
enough approximation, Jägersküpper's result indicates that
linear convergence can be expected. The same analysis applies to all
twice continuously differentiable unimodal function without critical
points other than the optimum.

\subsection{Saddle Points---The $p$-improvable Case}
\label{section:quadratic-saddle}

We consider the quadratic objective function
\begin{align*}
	f(x_1, x_2) := a \cdot x_1^2 - x_2^2
\end{align*}
with parameter $a > 0$. The origin is a saddle point. It is
$p$-improvable for all $p < 2 \cot^{-1}(\sqrt{a}) / \pi$ (see the
appendix for details). For small enough $a$, the success probability is
larger than $\tau$ and corollary~\ref{corollary:differentiable} applies,
while for large values of $a$ the success probability decays to zero and
we lose all guarantees.

Simulations show that the ES overcomes the zero level set containing the
saddle point without a problem, also for large values of $a$. It seems
that $p$-improvable saddle points do not result in premature convergence
of the algorithm, irrespective of the value of $p > 0$.
However, this statement is based on an empirical observation, not on a
rigorous proof.

\subsection{Saddle Points---The $p$-critical Case}
\label{section:cubic-saddle}

\begin{figure}
\begin{center}
	\includegraphics{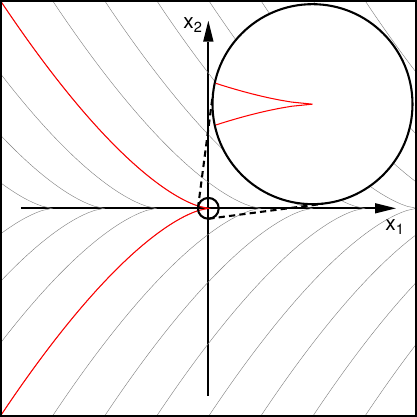}
	\caption{
		Level lines of the function $f(x_1, x_2) = x_1 + |x_2|^\frac{2}{3}$
		in the range $[-1, 1]^2$. The inset shows a zoom of factor $10$.
		\label{figure:x3y2}
	}
\end{center}
\end{figure}

The function
\begin{align*}
	f(x_1, x_2) = x_1 + |x_2|^\frac{2}{3}
\end{align*}
models a ridge shaped after the zero set of the cubic polynomial
$f(x_1, x_2) := x_1^3 + x_2^2$. It has $p$-critical saddle points on the
line $\R \times \{0\} \subset \R^2$ forming a ridge, see
figure~\ref{figure:x3y2}. Without loss of
generality we consider $m = 0 \in \R^2$ in the following. A successful
offspring $x \in \R^2$ fulfills $x_1^3 + x_2^2 \leq 0$. For small enough
$\sigma$ and hence for small enough $\|x\| \ll 1$,
$\|x\| \in \Theta(\sigma)$, this implies $-x_1 \gg |x_2|$ and hence
$-x_1 \in \Theta(\sigma)$ and $|x_2| \in o(\sigma)$. Plugging this into
the above inequality we obtain
$|x_2| \in \Order(-x_1 \cdot \sqrt{\sigma})$. Therefore, for small
$\sigma$ we have $p_f^{\leq}(0, \sigma) \in \Order(\sqrt{\sigma})$. This
implies that the cumulative success probability
\begin{align*}
	\sum_{t=0}^{\infty} p_f^{\leq}(0, e^{t \cdot c_-}) = \Order\left( \sum_{t=0}^\infty e^{t \cdot c_- / 2} \right) = \Order \left( \frac{1}{1 - e^{c_- / 2}} \right) = \Order(1)
\end{align*}
is finite, and theorem~\ref{theorem:premature-convergence} yields
(premature) convergence with arbitrarily high probability.

\subsection{Linear Ridge}
\label{section:linear-ridge}

Consider the linear ridge objective
\begin{align*}
	f(x_1, x_2) := x_1 + a \cdot |x_2|
\end{align*}
with parameter $a > 0$. The function is continuous, and its level sets
contain a kink. Again, the line $\R \times \{0\}$ is critical; this is
where the function is non-differentiable. The function is $p$-improvable
for $p < \cot^{-1}(a) / \pi < 1/2$ (see the appendix).
For $a \to \infty$ the success probability decays to zero.

As long as $\cot^{-1}(a) / \pi > \tau$ we can conclude divergence of the
algorithm (the intended behavior) from
theorem~\ref{theorem:global-convergence}. Otherwise we lose this
property, and it is well known and easy to check with simulations that
for large enough $a$ the algorithm indeed converges prematurely.

\subsection{Sphere with Jump}
\label{section:jump}

Our next example is an ``essentially discontinuous'' problem in the
sense that in general no function in the equivalence class $[f]$ is
continuous. We consider objective functions of the form
\begin{align*}
	f(x) := \|x\|^2 + \indicator_{S}(x),
\end{align*}
where $\indicator_S$ denotes the indicator function of a measurable
set~$S \subset \R^d$. If $S$ has a sufficiently simple shape then this
problem is similar to a constrained problem where $S$ is the infeasible
region \citep{arnold2008behaviour}, at least for small enough $\sigma$.
As long as $m^{(t)} \in S$ the (1+1)-ES essentially optimizes the sphere
function, and as soon as $m^{(t)} \not\in S$ the (soft) constraint comes
into play.

If $S$ is the complement of a star-shaped open neighborhood of the
origin then it is easy to see that the function is unimodal and
$p$-improvable for all $p < 1/2$.
Theorem~\ref{theorem:global-convergence} applied with $K_1 := \{0\}$
yields the existence of a sub-sequence converging to the origin, which
implies convergence of the whole sequence due to monotonicity of
$f \big( m^{(t)} \big)$. The results of \cite{jagerskupper2005rigorous}
and \cite{akimoto2018drift} imply linear convergence.

Other shapes of $S$ give different results. For example, for $d \geq 2$,
if $S$ is a ball not containing the origin then the function is still
unimodal. For example, define $S$ as the open ball of radius $1/2$
around the first unit vector $e_1 = (1, 0, \dots, 0) \in \R^d$. Then at
$m := 3/2 \cdot e_1$ we have $\xi_p^f(m) = 0$ for all $p > 0$, and
according to theorem~\ref{theorem:premature-convergence} the algorithm
can converge prematurely if the step size is small.
Alternatively, if $S$ is the closed ball, then all points except the
origin are $p$-improvable for all $p < 1/2$, however, there does not
exist a positive lower semi-continuous lower bound on $\xi_p^f$ in any
neighborhood of $m = 3/2 \cdot e_1$, and again the algorithm can
converge to this point, irrespective of the target success
probability~$\tau$.

Now consider the strip $S := (a, \infty) \times (0, 1) \subset \R^2$
with parameter $a > 0$. An elementary calculation of the success rate at
$m := (a + \varepsilon, 1)$ for $\sigma \to 0$ shows that the (1+1)-ES
is guaranteed to converge to the optimum irrespective of the initial
conditions if $\tan^{-1}(a) / (2 \pi) < \tau$ (details are found in the
appendix), i.e., if $a$ is large enough; otherwise the algorithm can
converge prematurely to a point on the edge $(a, \infty) \times \{1\}$
of~$S$.

\subsection{Extremely Rugged Barrier}
\label{section:fat-cantor}

Let us drive the above discontinuous problem to the extreme.
Consider the one-dimensional problem
\begin{align*}
	f(x) := x + \indicator_S(x),
\end{align*}
where $S \subset [-1, 0]$ is a Smith-Volterra-Cantor set, also known as
a fat Cantor set. $S$ is closed, has positive measure (usually chosen as
$\Lambda(S) = 1/2$), but is nowhere dense. Counter-intuitively, the
function is unimodal in the sense that no point is optimal restricted to
an open neighborhood (which is what commonly defines a local optimum).
Still, intuitively, $S$ should act as a barrier blocking optimization
progress with high probability.

The function is point-wise $p$-improvable everywhere. However, similar
to the closed ball case in the previous section, there is no positive,
lower semi-continuous lower bound on $\xi_p^f$. Therefore
theorem~\ref{theorem:global-convergence} does not apply.
Indeed, unsurprisingly, simulations%
\footnote{Special care must be taken when simulating this problem with
  floating point arithmetic. Our simulation is necessarily inexact,
  however, not beyond the usual limitations of floating point numbers.
  It does reflect the actual dynamics well. The fitness function is
  designed such that the most critical point for the simulation is zero,
  which is where standard IEEE floating point numbers have maximal
  precision.}
show that the algorithm gets stuck with positive probability when
initialized with $0 < x^{(0)} \ll 1$ and $\sigma \ll 1$.
When removing $0$ from $S$, then analogous to
section~\ref{section:cubic-saddle} we obtain
$p_f^{\leq}(m, \sigma) \in \Order(\sqrt{\sigma})$ for $m = 0$ and small
$\sigma$, and hence theorem~\ref{theorem:premature-convergence} applies.

In contrast, if $S$ is a Cantor set of measure zero then the algorithm
diverges successfully, since it ignores zero sets with full probability.

%%%%%%%%%%%%%%%%%%%%%%%%%%%%%%%%%%%%%%%%%%%%%%%%%%%%%%%%%%%%
%%
\section{Conclusions and Future Work}
\label{section:conclusion}

We have established global convergence of the (1+1)-ES for an extremely
wide range of problems. Importantly, with the exception of a few proof
details, the analysis captures the actual dynamics of the algorithm and
hence consolidates our understanding of its working principles.

Our analysis rests on two pillars. The first one is a progress guarantee
for rank-based evolutionary algorithms with elitist selection. In its
simplest form, it bounds the progress on problems without plateaus from
below. It seems to be quite generally applicable, e.g., to runtime
analysis and hence to the analysis of convergence speed.

The second ingredient is an analysis of success-based step size control.
The current method barely suffices to show global convergence. It is not
suitable for deducing stronger statements like linear convergence on
scale invariant problems. Control of the step size on general problems
therefore needs further work.

Many natural questions remain open, the most significant are listed in
the following. These open points are left for future work.
\begin{itemize}
\item
	The approach does not directly yield results on the speed of
	convergence. However, the progress guarantee of
	theorem~\ref{theorem:decrease} is a powerful tool for
	such an analysis. It can provide us with drift conditions and hence
	yield bounds on the expected runtime and on the tails of the runtime
	distribution. But for that to be effective we need better tools for
	bounding the tails of the step size distribution. Here, again, drift
	is a promising tool.
\item
	The current results are limited to step-size adaptive algorithms and
	do not include covariance matrix adaptation. One could hope to
	extend the proceeding to the (1+1)-CMA-ES
	algorithm~\citep{igel2007covariance}, or to (1+1)-xNES
	\citep{glasmachers:2010c}. Controlling the stability of the
	covariance matrix is expected to be challenging. It is not clear
	whether additional assumptions will be required. As an added
	benefit, it may be possible to relax the condition $p > \tau$ for
	$p$-improvability, by requiring it only after successful adaptation
	of the covariance matrix.
\item
	Plateaus are currently not handled.
	Theorem~\ref{theorem:decrease} shows how they distort
	the distribution of the decrease. Worse, they affect step size
	adaptation, and they make it virtually impossible to obtain a lower
	bound on the one-step probability of a strict improvement.
	Therefore, proper handling of plateaus requires additional arguments.
\item
	In the interest of generality, our convergence theorem only
	guarantees the existence of a limit point, not convergence of the
	sequence as a whole. We believe that convergence actually holds in
	most cases of interest (at least as long as there are no plateaus,
	see above). This is nearly trivial if the limit point is an isolated
	local optimum, however, it is unclear for a spatially extended
	optimum, e.g., a low-dimensional variety or a Cantor set.
\item
	Our current result requires a saddle point to be $p$-improvable for
	some $p > \tau$, otherwise the theorem does not exclude convergence
	of the ES to the saddle point. We know from simulations that the
	(1+1)-ES overcomes $p$-improvable saddle points reliably, also for
	$p \ll \tau$. A proper analysis guaranteeing this behavior would
	allow to establish statements analogous to work on gradient-based
	algorithms that overcome saddle points quickly and reliably, see
	e.g.\ \cite{dauphin2014identifying}. However, this is clearly beyond
	the scope of the present paper.
\item
	We provide only a minimal negative result stating that the algorithm
	may indeed converge prematurely with positive probability if there
	exists a $p$-critical point for which the cumulative success
	probability does not sum to infinity. In section~\ref{section:jump}
	it becomes apparent that this notion is rather weak, since the
	statement is not formally applicable to the case of a closed ball,
	which however differs from the open ball scenario only on a zero
	set. This makes clear that there is still a gap between positive
	results (global convergence) and negative results (premature
	convergence). Theorem~\ref{theorem:premature-convergence} can for
	sure be strengthened, but the exact conditions remain to be
	explored. A single $p$-improvable point with $p < \tau$ is
	apparently insufficient. A $p$-critical point may be sufficient, but
	it is not necessary.
\end{itemize}

\subsubsection*{Acknowledgments}

I would like to thank Anne Auger for helpful discussions,
and I gratefully acknowledge support by Dagstuhl seminar 17191
``Theory of Randomized Search Heuristics''.

{
\small
\bibliographystyle{apalike}
%\bibliography{global-convergence}

}

%%%%%%%%%%%%%%%%%%%%%%%%%%%%%%%%%%%%%%%%%%%%%%%%%%%%%%%%%%%%
%%
%\newpage
\appendix
\section*{Appendix}

Here we provide the proofs of technical lemmas that were omitted from
the main text in the interest of readability.

\vspace*{1em}\hrule\vspace*{1em}

\begin{proof}[Proof of lemma~\ref{lemma:equivalence}]
We have to show that the level sets of all three functions agree outside
a set of measure zero.
It is immediately clear from definition~\ref{definition:equivalence}
that the level sets of $f$ are a refinement of the level sets of
$\f^{\leq}$ and $\f^{<}$, i.e., $f(x) = f(x')$
implies $\f^{\leq}(x) = \f^{\leq}(x')$ and
$\f^{<}(x) = \f^{<}(x')$, and
$\f^{\leq}(x) < \f^{\leq}(x')$ and
$\f^{<}(x) < \f^{<}(x')$ both imply $f(x) < f(x')$.

It remains to be shown that $\f^{\leq}$ and $\f^{<}$
do not join $f$-level sets of positive measure.
Let $y \in \R$ denote a level so that $Y = \big(\f^{<}\big)^{-1}(y)$
has positive measure $\Lambda(Y) > 0$. We have to show that this measure
(not necessarily the whole set, only up to a zero set) is covered by a
single $f$-level set.
Assume the contrary, for the sake of contradiction. Then we find
ourselves in one of the following situations:
\begin{enumerate}
\item
	There exist $x, x' \in Y$ fulfilling $a := f(x) < f(x') =: a'$ and
	it holds $\Lambda\big(f^{-1}(a)\big) > 0$ and
	$\Lambda\big(f^{-1}(a')\big) > 0$. So the mass of $Y$ is split
	into at least two chunks of positive measure. This implies
	$\f^{<}(x') - \f^{<}(x) \geq \Lambda\big(f^{-1}(a)\big) > 0$,
	which contradicts the assumption that $x$ and $x'$ belong to the same
	$\f^{<}$-level.
\item
	There exist $x, x' \in Y$ fulfilling $a = f(x) < f(x') = a'$ and it
	holds $\Lambda\big(f^{-1}(I)\big) > 0$ for the open interval
	$I = (a, a')$. So $Y$ consists of a continuum of level sets of
	measure zero. Again, this implies
	$\f^{<}(x') - \f^{<}(x) \geq \Lambda\big(f^{-1}(I)\big) > 0$,
	leading to the same contradiction as in the first case.
\end{enumerate}
The argument for $\f^{\leq}$ is exactly analogous.
\end{proof}

\vspace*{1em}\hrule\vspace*{1em}

\begin{proof}[Proof of lemma~\ref{lemma:increasing-step-size}]
It holds
\begin{align*}
	p_f^{<}(m, a \cdot \sigma)
		&= \int_{S_f^{<}(m)} \frac{1}{(2\pi)^{d/2} a^d \sigma^d} \cdot \exp \left( -\frac{\|x - m\|^2}{2 a^2 \sigma^2} \right) \,dx \\
		&\geq \frac{1}{a^d} \cdot \int_{S_f^{<}(m)} \frac{1}{(2\pi)^{d/2} \sigma^d} \cdot \exp \left( -\frac{\|x - m\|^2}{2 \sigma^2} \right) \,dx \\
		&= \frac{1}{a^d} \cdot p_f^{<}(m, \sigma).
\end{align*}
The computation for $p_f^{\leq}$ is analogous.
\end{proof}

\vspace*{1em}\hrule\vspace*{1em}

\begin{proof}[Proof of lemma~\ref{lemma:gap}]
Fix $x$ and define $\xi := \xi_{p_T}^f(x)$.
The cases $p_H = 0$ and $\xi = 0$ are trivial, so in the following we
treat the case that both are positive. For $a \geq 1$ it holds
\begin{align*}
	p_T &= \phantom{a^d \cdot} \int_{S_f^{<}(x)} \frac{1}{(2\pi)^{d/2} \xi^d} \exp\left( -\frac{\|x' - x\|^2}{2 \xi^2} \right) \,dx' \\
	       &=    a^d \cdot \int_{S_f^{<}(x)} \frac{1}{(2\pi)^{d/2} a^d \xi^d} \exp\left( -\frac{\|x' - x\|^2}{2 \xi^2} \right) \,dx' \\
	       &\leq a^d \cdot \int_{S_f^{<}(x)} \frac{1}{(2\pi)^{d/2} a^d \xi^d} \exp\left( -\frac{\|x' - x\|^2}{2 a^2 \xi^2} \right) \,dx'.
\end{align*}
In other words, the success probability for step size $a \cdot \xi$
is at least $p_T / a^d$. Hence, in order to push the success
probability below $p_T / a^d$, the step size must be at least
$\xi \cdot a$, which therefore bounds $\eta_{p_T / a^d}^f(x)$ from
below. Applying the above argument with
$a = \sqrt[d]{p_T / p_H}$ completes the proof.
\end{proof}

\vspace*{1em}\hrule\vspace*{1em}

\begin{proof}[Proof of lemma~\ref{lemma:differentiable-improvable}]
In a small enough neighborhood of a regular point $x$ the function $f$
can be approximated arbitrarily well by a linear function (its first
order Taylor polynomial). In particular, the level set of $f$ is
arbitrarily well approximated by a hyperplane, for which the
probability of strict improvement is exactly $1/2$. Hence we have
\begin{align*}
	\lim_{\sigma \to 0} p_f^{<}(x, \sigma) = \frac{1}{2},
\end{align*}
which immediately implies the first statement.

We have already seen that the second statement holds point-wise. It
remains to be shown that $\xi_p^f|_Y$ is lower bounded by a positive,
lower semi-continuous function. To this end we show that $\xi_p^f$
itself is lower-semi-continuous, and we note that $\xi_p^f|_Y$ takes
positive values.
Consider a convergent sequence $(a_t)_{t \in \N} \to x \in \R^d$ and
define $\xi_a := \lim\inf_{t \to \infty} \xi_p^f(a_t)$ and
$\xi_x := \xi_p^f(x)$. We have to show that it holds $\xi_x \leq \xi_a$
for all choices of $x$ and $(a_t)_{t \in \N}$. We define
\begin{align*}
	S_x :=\, & \Big\{ \sigma \in \R^+ \,\Big|\, p^<_f(x, \sigma) \leq p \Big\} \\
	\text{and} \qquad S_a :=\, & \Big\{ \sigma \in \R^+ \,\Big|\, \exists (t_k)_{k \in \N} :\, p^<_f(a_{t_k}, \sigma) \leq p \,\forall k \in \N \Big\},
\end{align*}
which allows us to write $\xi_a = \inf(S_a)$ and $\xi_x = \inf(S_x)$.
Fix $\sigma \in S_a$ and a corresponding sub-sequence $(t_k)_{k \in \N}$
so that it holds $p^<_f(a_{t_k}, \sigma) \leq p \,\forall k \in \N$.
From the continuity of $f$ it follows that the success probability
function $p_f^{<}$ is lower semi-continuous (and even continuous in its
second argument, the step size). From $\lim_{k \to \infty} a_{t_k} = x$
and lower semi-continuity of $p^<_f$ it follows $\sigma \in S_x$. We
conclude $S_a \subset S_x$ and therefore $\xi_x \leq \xi_a$.

%We have to show that $(\xi_p^f)^{-1}([0, r])$ is
%closed for all $r > 0$. Due to
%$\lim_{\sigma \to 0}p_f^{<}(x, \sigma) = 1/2 > p$, this is the case if
%$(p_f^{<})^{-1}([0, q])$ is closed for all $q \in [0, 1/2]$. This again
%follows from the fact that for continuous $f$ the function $p_f^{<}$ is
%lower semi-continuous, and even continuous in its second argument (the
%step size).

To show the last statement we construct a cone of improving steps
centered at $x$. This cone makes up a fixed fraction of each ball
centered on $x$, which shows that $x$ is $p$-improvable, where $p$ is
any number smaller than the volume of the intersection of ball and cone
divided by the volume of the ball, which is well-defined and positive in
the limit when the radius tends to zero.
Let $v$ denote an eigen vector of $H$ fulfilling $v^T H v < 0$. For
$\sigma \to 0$, the objective function is well approximated by the
quadratic Taylor expansion
\begin{align*}
	f(x') \approx g(x') = f(x) + (x - x')^T H (x - x').
\end{align*}
The sub-level set $S_f^{<}(x)$ is locally well approximated by
$S_g^{<}(x)$, which is a cone centered on $x$. Whether a ray
$x + \R \cdot z$ belongs to $S_g^{<}(x)$ or not depends on whether
$z^T H z < 0$ or not. Now, the eigen vector $v$ has this property, and
due to continuity of $g$, the same holds for an open neighborhood $N$ of
$v$. The cone $x + \R \cdot N$ is contained in $S_g^{<}(x)$ and has the
same positive probability $s_g^{<}(x, \sigma) = p > 0$ under
$\Normal(x, \sigma^2 I)$ for all $\sigma > 0$. We conclude
\begin{align*}
	\lim_{\sigma \to 0} p_f^{<}(x, \sigma) \geq p > 0,
\end{align*}
which completes the proof.
\end{proof}

\vspace*{1em}\hrule\vspace*{1em}

\begin{proof}[Proof of lemma~\ref{lemma:upper-bound}]
We use the short notation $m := m^{(t)}$ and $\sigma = \sigma^{(t)}$.
Let $S = S_f^{<}(m)$ denote the region of improvement, with Lebesgue
measure $\f(m)$. The probability of sampling from this region is bounded
by
\begin{align*}
	p_f^{<}(m)
			&= \int_S \frac{1}{(2\pi)^{d/2} \sigma^d} \exp \left( -\frac{\|x-m\|^2}{2 \sigma^2} \right) \, \text{d}x \\
			&= \frac{1}{(2\pi)^{d/2} \sigma^d} \int_S \exp \left( -\frac{\|x-m\|^2}{2 \sigma^2} \right) \, \text{d}x \\
			&< \frac{1}{(2\pi)^{d/2} \sigma^d} \int_S \text{d}x \\
			&= \frac{\f(m)}{(2\pi)^{d/2} \sigma^d} \\
			&\leq p,
\end{align*}
where the last inequality is equivalent to the assumption.
\end{proof}

\vspace*{1em}\hrule\vspace*{1em}

\begin{proof}[Proof of lemma~\ref{lemma:subsequence}]
Assume the contrary, for the sake of contradiction. Then
\begin{align*}
	\sum_{t=1}^\infty X^{(t)} < \infty.
\end{align*}
Fix $N \in \N$. Hoeffding's inequality applied with $\varepsilon = p/2$
and $n \geq \frac{2N}{p}$ yields
\begin{align*}
	\Pr \left( \sum_{t=1}^n X^{(t)} \leq N \right) \leq \exp \left( -n \cdot \frac{p^2}{2} \right) \underset{n \to \infty}{\longrightarrow} 0.
\end{align*}
Hence, for $n \to \infty$, with full probability the infinite sum
exceeds $N$. Since $N$ was arbitrary, we arrive at a contradiction.
\end{proof}

\vspace*{1em}\hrule\vspace*{1em}

\begin{proof}[Proof of lemma~\ref{lemma:sigma-in-range}]
In each iteration, the step size $\sigma$ is multiplied by either
$e^{c_-}$ or $e^{c_+}$. According to Lemma~\ref{lemma:gap},
the condition $\frac{p_H}{p_T} \leq e^{d \cdot c_-}$ yields
\begin{align*}
	\frac{\eta_{p_H}^f \big( m^{(t_k)} \big)}{\xi_{p_T}^f \big( m^{(t_k)} \big)} \geq e^{-c_-}.
\end{align*}
An unsuccessful step of the (1+1)-ES in iteration $t$ results in a
reduction of the step size by the factor
$\frac{\sigma^{(t+1)}}{\sigma^{(t)}} = e^{c_-} < 1$ and leaves
$m^{(t+1)} = m^{(t)}$ unchanged. We conclude that in no such step can
overjump the interval
$\big[ \xi_{p_T}^f \big( m^{(t)} \big), \eta_{p_H}^f \big( m^{(t)} \big) \big]$,
in the sense of $\sigma^{(t)} \geq \eta_{p_H}^f \big( m^{(t)} \big)$ and
$\sigma^{(t+1)} \leq \xi_{p_T}^f \big( m^{(t)} \big)$.
The above property also implies $\frac{b_H}{b_T} \geq e^{-c_-}$.

The central proof argument works as follows. First we exclude that the
step size remains outside $[b_T, b_H]$ for too long. The same argument
does not work for the target interval defined in
equation~\eqref{eq:sigma-in-range} because of its time dependency---we
could overjump the moving target. Instead we show that the only way for
the step size to avoid the target interval for an infinite time is to
overjump, i.e., to find itself above and below the interval infinitely
often. Finally, an argument exploiting the properties of unsuccessful
steps allows us to consider a static target, which cannot be overjumped
by the property already shown above.

First we show that there exists an infinite sub-sequence of iterations
$t$ fulfilling $\sigma^{(t)} \in [b_T, b_H]$. This statement is strictly
weaker than the assertion to be shown. It is still helpful in the
following because then we know that the step sizes returns to a fixed,
$t$-independent interval for an infinite number of times.
Assume for the sake of contradiction that there exists $t_0$ such that
$\sigma^{(t)} \leq b_T$ for all $t \geq t_0$. The logarithmic step size
change $\delta^{(t)} := \log(\sigma^{(t+1)}) - \log(\sigma^{(t)})$ takes
the values $c_+ > 0$ with probability at least $p_T > \tau$ and
$c_- < 0$ with probability at most $1 - p_T < 1 - \tau$, hence
\begin{align*}
	\Expectation\big[\delta^{(t)}\big] \geq \Delta := p_T \cdot c_+ + (1 - p_T) \cdot c_- > 0.
\end{align*}
For $t_1 > t_0$ we consider the random variable
$\log(\sigma^{(t_1)}) = \log(\sigma^{(t_0)}) + \sum_{t=t_0}^{t_1-1} \delta^{(t)}$.
The variables $\delta^{(t)}$ are not independent. We create independent
variables as follows. For each candidate state $(m, \sigma)$ fulfilling
$\sigma < b_T$ we fix a set $I(m, \sigma) \subset S_f^{<}(m)$ of improving
steps with probability mass exactly $p_T$ under the distribution
$\Normal(m, \sigma^2 I)$. Let $\tilde\delta^{(t)}$ denote the step size
change corresponding to $\delta^{(t)}$ for which the step size is
increased only if the iterate $m^{(t+1)}$ is contained in $I(m,
\sigma)$. Note that these hypothetical step size changes do not
influence the actual sequence of algorithm states. Therefore the
sequence is i.i.d., and it holds $\tilde\delta^{(t)} \leq \delta^{(t)}$.
From Hoeffding's inequality applied with $\varepsilon = \Delta/2$ to
$\sum_{t = t_0}^{t_1 - 1} \tilde\delta^{(t)} \leq \sum_{t = t_0}^{t_1 - 1} \delta^{(t)}$
we obtain
\begin{align*}
	\Pr & \left\{ \log(\sigma^{(t_1)}) \leq \log(\sigma^{(t_0)}) + (t_1 - t_0) \cdot \frac{\Delta}{2} \right\} \\
		& \leq \exp \left( - (t_1 - t_0) \cdot \frac{\Delta^2}{2 (c_+ - c_-)^2} \right),
\end{align*}
i.e., the probability that the log step size grows by less than
$\Delta/2$ per iteration on average is exponentially small in
$t_1 - t_0$. For
$t_1 \gg t_0 + 2 / \Delta \cdot \big(\log(b_T) - \log(\sigma^{(t_0)})\big)$
the probability becomes minuscule, and for $t_1 \to \infty$ it
vanishes completely. Hence, with full probability, we arrive at a
contradiction. The same logic contradicts the assumption that
$\sigma^{(t)} \geq b_H$ for all $t \geq t_0$. Hence, with full
probability, sub-episodes of very small and very large step size are of
finite length, and according to lemma~\ref{lemma:subsequence} the
sequence of step sizes returns infinitely often to the interval
$[b_T, b_H]$.

Next we show that there exists an infinite sub-sequence of iterations
fulfilling equation~\eqref{eq:sigma-in-range}. Again, assume the
contrary. We know already that $\sigma^{(t)}$ does not stay below $b_T$
or above $b_H$ for an infinite time. Hence, there must exist an infinite
sub-sequence fulfilling either
\begin{align}
	\sigma^{(t)} \in \Big[ b_T, \xi_{p_T}^f \big( m^{(t)} \big) \Big]
	\label{eq:rather-small}
\end{align}
or
\begin{align}
	\sigma^{(t)} \in \Big[ \eta_{p_H}^f \big( m^{(t)} \big), b_H \Big].
	\label{eq:rather-big}
\end{align}
Assume an infinite sub-sequence fulfilling equation~\eqref{eq:rather-small}.
For each of these iterations, the success probability is lower bounded
by $p_T$. Consider the case of consecutive successes. Until the event
\begin{align}
	\sigma^{(t)} \geq \xi_{p_T}^f \big( m^{(t)} \big)
	\label{eq:exceeds-lower}
\end{align}
the probability of success remains lower bounded by $p_T > 0$. The
condition is fulfilled after at most
$n^+ := \big(\log(b_H) - \log(b_T)\big) \big/ c_+$ successes in a row,
hence the probability of such an episode occurring is lower bounded by
$p_T^{n^+} > 0$. Lemma~\ref{lemma:subsequence} ensures the existence of
an infinite sub-sequence of iterations with this property. Each such
episode contains a point fulfilling either
equation~\eqref{eq:sigma-in-range} or equation~\eqref{eq:exceeds-lower}.
By assumption, the former happens only finitely often, which implies
that the latter happens infinitely often.

Hence, this case as well as the alternative assumption of an infinite
sequence fulfilling equation~\eqref{eq:rather-big}, handled with an
analogous argument, result in an infinite sub-sequence with the property
\begin{align*}
	\sigma^{(t)} \in \Big[ \eta_{p_H}^f \big( m^{(t)} \big), e^{c_+} \cdot b_H \Big].
\end{align*}
Following the same line of arguments as above, as long as
$\sigma^{(t)} \geq \eta_{p_H}^f \big( m^{(t)} \big)$, the probability of
an unsuccessful step is lower bounded by $1 - p_H > 0$. After at most
$n^- := \big(\log(b_T) - \log(b_H) + c_+\big) \big/ c_-$ unsuccessful
steps in a row, called an episode in the following, the step size must
have dropped below $b_T \leq \eta_{p_H}^f \big( m^{(t)} \big)$, hence
the probability of such an episode occurring is lower bounded by
$(1 - p_H)^{n^-} > 0$. According to lemma~\ref{lemma:subsequence}, an
infinite number of such episodes occurs.

By construction, these episodes consist entirely of unsuccessful steps,
and therefore $m^{(t)}$ remains unchanged for the duration of an
episode. This comes handy, since this means that also the target
interval
$\Big[ \xi_{p_T}^f \big( m^{(t)} \big), \eta_{p_H}^f \big( m^{(t)} \big) \Big]$
remains fixed, and this again means that at least one iteration of the
episode falls into this interval. We have thus constructed an infinite
sub-sequence of iteration within the above interval, in contradiction to
the assumption.
\end{proof}

\vspace*{1em}\hrule\vspace*{1em}

Finally, we provide details on the computations of success rates in the
examples. In section~\ref{section:quadratic-saddle}, the set where the
function $f(x_1, x_2) := a \cdot x_1^2 - x_2^2$ takes the value zero
consists of two lines through the origin in directions $(1, \sqrt{a})$
and $(-1, \sqrt{a})$. The cone bounded by these lines in the success
domain. The angle between their directions divided by $\pi$ corresponds
to the success rate. It is two times the angle between $(1, \sqrt{a})$
and $(1, 0)$, and hence $2 \cot^{-1}(\sqrt{a})$. Dividing by $\pi$
yields the result.

The threshold $p < \cot^{-1}(a) / \pi$ in
section~\ref{section:linear-ridge} follows the exact same logic, with
the difference that the square root vanishes in the direction vectors,
and we lose a factor of two, since the success domain is only one half
of the cone.

In section~\ref{section:jump}, the circular level line in the corner
point $(a, 1)$ is tangent to the vector $(-1, a)$. The angle
$\tan^{-1}(a)$ between $(-1, a)$ and $(-1, 0)$, divided by $2 \pi$, is a
lower bound on the success rate at $m = (a + \varepsilon, 1)$ with
$\sigma \ll \varepsilon$. The bound is precise for $\varepsilon \to 0$.

\end{document}